\documentclass{article}
\usepackage{spconf,amsmath,graphicx,hyperref}

\usepackage[noadjust]{cite}
\usepackage{amsmath, amssymb, amsfonts}
\usepackage{algorithmic}
\usepackage[ruled,linesnumbered]{algorithm2e}
\usepackage{graphicx, subcaption}
\usepackage{textcomp}
\usepackage{xcolor}

\usepackage{mathtools}
\usepackage[colorinlistoftodos]{todonotes}

\input{mysymbol.sty}

\usepackage{amsthm}

\newtheorem{proposition}{\hspace{0pt}\bf Proposition}

\newtheorem{remark}{\hspace{0pt}\it \underline{Remark}}

\usepackage{listings}
\usepackage{enumerate}
\usepackage{enumitem}
\usepackage{bbm}

\colorlet{LightGreen}{white!70!green}

\colorlet{LightBlue}{white!70!blue}

\title{Graph-Aware Diffusion for Signal Generation}
\ninept

\name{\small Sergio Rozada$^*$, Vimal K B$^\dagger$, Andrea Cavallo$^\dagger$, Antonio G. Marques\sthanks{Work partially supported by the Spanish AEI (10.13039/501100011033) grant PID2022-136887NB-I00, and the Community of Madrid via the Ellis Madrid Unit and grants URJC/CAM F1180 and TEC-2024/COM-89.}, Hadi Jamali-Rad$^\dagger$, Elvin Isufi\sthanks{Part of this work was supported by the TU Delft AI Labs programme, the NWO OTP GraSPA proposal \#19497, the NWO VENI proposal 222.032, and the TU Delft - Shell.ai TKI project.}\vspace{-0.4cm}}
\address{\small $^*$ King Juan Carlos University, Madrid, Spain \\
\small $^\dagger$ Delft University of Technology, Delft, Netherlands}

\begin{document}
\ninept

\maketitle
\begin{abstract}
    We study the problem of generating graph signals from unknown distributions defined over given graphs, relevant to domains such as recommender systems or sensor networks. 
    Our approach builds on generative diffusion models, which are well established in vision and graph generation but remain underexplored for graph signals. 
    Existing methods lack generality, either ignoring the graph structure in the forward process or designing graph-aware mechanisms tailored to specific domains.
    We adopt a forward process that incorporates the graph through the heat equation. 
    Rather than relying on the standard formulation, we consider a time-warped coefficient to mitigate the exponential decay of the drift term, yielding a graph-aware generative diffusion model (GAD).
    We analyze its forward dynamics, proving convergence to a Gaussian Markov random field with covariance parametrized by the graph Laplacian, and interpret the backward dynamics as a sequence of graph-signal denoising problems. 
    Finally, we demonstrate the advantages of GAD on synthetic data, real traffic speed measurements, and a temperature sensor network.
\end{abstract}


\vspace{-0.1cm}

\section{Introduction}

\vspace{-0.1cm}

Data defined on irregular domains is pervasive, spanning applications from recommender systems to sensor networks, or wireless networks \cite{gao2023survey, chien2024opportunities}.
In this work, we focus on sampling from data distributions defined on irregular domains when the distribution itself is unavailable.
To that end, we leverage generative diffusion models (DMs), which have emerged as powerful tools for sampling unknown distributions \cite{yang2023diffusion}.
These models transform samples from an initial data distribution into noise through a (forward) diffusion process, typically converging to a simple stationary distribution.
By reversing this process, we have a (backward) mechanism that maps noisy samples back into the original data distribution.


While generative DMs have been proposed for graph generation \cite{liu2023generative}, less attention has been paid to generating signals on a known graph.
Existing approaches typically incorporate the graph only in the backward process, keeping the forward process graph-agnostic \cite{uslu2025generative}.
Yet, when signals are defined on graphs, the graph heat diffusion \cite{kondor2002diffusion} offers a natural way to design a graph-aware forward–backward process.
This idea has recently been explored for generative tasks in image synthesis \cite{rissanen2023generative} and Schr\"odinger bridges \cite{yang2025topological}.
Nonetheless, existing works remain restricted to the standard heat equation, which injects noise too rapidly and limits its suitability for generative purposes. In addition, they neither characterize its limiting behavior nor examine the influence of the graph on this limit.


In this work, we introduce GAD, a graph-aware DM for graph signal generation that addresses these gaps. 
We provide a comprehensive study of the forward and backward processes, highlighting the unique subtleties of applying DMs to graph signal generation. Our contributions are
\begin{itemize}[leftmargin=15pt,labelsep=0.5em, itemsep=0.0em]
    \item[\textbf{C1}] We consider a heat equation with a time-warped drift to control the decay of graph Laplacian modes and the pace of noise injection, ensuring smooth forward and backward processes.
    \item[\textbf{C2}] We analyze the stationary distribution induced by the proposed heat equation, which corresponds to a Gaussian Markov random field (GMRF) with covariance parametrized by the Laplacian.
    \item[\textbf{C3}] We connect the backward process to graph-signal processing (GSP), showing that it reduces to a sequence of graph denoising problems with a noise model that applies a graph filter to the signal and injects GMRF noise with known covariance.
    \item[\textbf{C4}] We demonstrate the advantages of our method on (i) a synthetic community-based setup, (ii) a real temperature sensor network, and (iii) real traffic speed measurements.
\end{itemize}

\noindent \textbf{Related work.} DMs have been applied to generate graphs \cite{liu2023generative}, but less attention has been paid to the case where the graph is given. 
In the latter setting, most methods incorporate the graph only in the backward process \cite{chen2025dhmoe, daiya2024diffstock, uslu2025generative}. 
A more principled alternative is to introduce the graph already in the forward process through the heat equation, which has been highly successful in graph machine learning \cite{kondor2002diffusion, vishwanathan2010graph} and has recently proven effective for inpainting tasks \cite{choi2023blurring, zhu2024graph, xia2025s, xu2025point, wang2025doublediffusion}.
For generation, graph-aware diffusion based on heat dynamics has been explored in computer vision \cite{hoogeboom2023blurring, rissanen2023generative} and applied to atomic position modeling \cite{park2024equivariant} and topological Schr\"odinger-bridges via optimal transport \cite{yang2025topological}.
However, these works (i) remain limited to the standard heat equation, which noises the signal too quickly and hinders generation, (ii) do not characterize the limiting behavior or the role of the graph in it, and (iii) provide little insight into the role of the graph in the backward process.
In contrast, we adapt the heat equation to control the noising rate, analyze its convergence to a graph-aware stationary distribution, and show that the backward process reduces to a sequence of graph-signal denoising problems.



\section{Preliminaries}

\textbf{GSP.} A graph $\ccalG = (\ccalV, \ccalE)$ consists of a set of $N$ nodes $\ccalV$ and edges $\ccalE \subseteq \ccalV \times \ccalV$.
Common algebraic representations of the graph $\ccalG$ include the adjacency matrix $\bbA \in \mathbb{R}^{N \times N}$ and the normalized graph Laplacian $\bbL = \bbI - \bbD^{-\frac{1}{2}} \bbA \bbD^{-\frac{1}{2}}$, where $\bbD=\diag(\bbA \bbone)$. 
A graph signal is a function defined on the set of nodes, represented as a vector $\bbx \in \mathbb{R}^N$, where $x_i$ denotes the signal value at node $i$.
The graph Laplacian $\bbL$ is diagonalizable as $\bbL = \bbV \bbLambda \bbV^\top$, where $\bbV \in \mathbb{R}^{N \times N}$ contains the eigenvectors of $\bbL$, and $\bbLambda$ is a diagonal matrix of corresponding eigenvalues.
The matrix $\bbV^{\top}$ defines the graph Fourier transform (GFT), so that $\tbx = \bbV^\top \bbx$ is the graph frequency representation of the signal $\bbx$~\cite{ortega2018graph}.
Graph convolutional filters are topology-aware operators for processing graph signals that are matrix polynomials of the Laplacian $\bbL$~\cite{isufi2024graph}, namely
\begin{equation}
    \label{eq:graph_filter}
     h(\bbL) = \sum_{k=0}^{K} \theta_k \bbL^k = \bbV \left( \sum_{k=0}^{K} \theta_k \bbLambda^k \right) \bbV^\top,
\end{equation}
where $\bbtheta = [\theta_0, \dots, \theta_{K}]^\top$ are the filter coefficients.
Since $\bbL^k$ captures $k$-hop neighborhoods, the output $h(\bbL)\bbx$ represents a diffusion of $\bbx$ across the graph, with $\theta_k$ controlling the weight of each $k$-hop.

\noindent \textbf{Generative DMs.} We consider that graph signals are random and drawn from an unknown distribution $p_0$. Although $p_0$ is not explicitly known, we have access to i.i.d. samples $\{ \bbx_0^{(i)}\}_{i=1}^M \sim p_0$. 
The goal is to learn to generate new samples from $p_0$.
To this end, we adopt a generative DM perspective.
In DMs, a forward process maps a signal $\bbx_0$ into Gaussian noise.
Then, by approximating the inverse of this mapping we can recover a signal from $p_0$ starting from Gaussian noise.
Specifically, the data-to-noise mapping is defined by a forward diffusion stochastic differential equation (SDE) of the form
\begin{equation}
    \label{eq::diff_raw}
    d\bbx_t = f(\bbx_t, t)dt + g(\bbx_t, t)d\bbw_t,
\end{equation}
where $f$ is the deterministic drift term,  $g$ is the diffusion coefficient, and $\bbw_t$ is Brownian motion \cite{baxendale2007stochastic}.
By appropriately designing $f$ and $g$, the process converges at terminal time $T$ to a Gaussian distribution.
To approximate the inverse mapping, we rely on the fact that diffusion SDEs can be reversed under mild conditions \cite{anderson1982reverse}, yielding the backward SDE
\begin{equation}
    \label{eq::diff_back_raw}
    d\bbx_t \!=\! \left[f(\bbx_t, t) - g(\bbx_t, t)^2 \nabla_{\bbx_t} \log p_t(\bbx_t)\right] dt + g(\bbx_t, t)\,d\bbw_t,
\end{equation}
where $p_t$ denotes the data distribution at timestep $t$, and the gradient $\nabla_{\bbx_t} \log p_t(\bbx_t)$ is termed as score function.
Since $f$ and $g$ are predefined, if the score function $\nabla_{\bbx_t} \log p_t(\bbx_t)$ was known, the reverse process in \eqref{eq::diff_back_raw} could be run from a sample of $p_T$ to generate a sample from $p_0$.
Unfortunately, this is rarely the case, as the marginal distribution $p_t$ depends on  $p_0$.
This leads to the fundamental problem of DMs, which is to estimate the scores. 
As shown in the following sections, this task naturally relates to graph-signal denoising.

\section{Graph-aware Generative Diffusion Model}

GAD leverages the graph $\ccalG$ to design a graph-aware SDE for sampling from $p_0$, which requires specifying the drift $f$ and the diffusion coefficient $g$ in \eqref{eq::diff_raw}, as well as a methodology to approximate the score function $\nabla_{\bbx_t} \log p_t(\bbx_t)$ in \eqref{eq::diff_back_raw}.

\noindent \textbf{Forward process.} 
We model the forward dynamics as the heat equation on the graph Laplacian $\bbL$ with additive noise, which is a (time-inhomogeneous) Ornstein–Uhlenbeck process given by
\begin{equation}
    \label{eq::forward}
    d\bbx_t = -c_t \bbL_\gamma \bbx_t\,dt + \sqrt{2c_t}\,\sigma\,d\bbw_t,
\end{equation}
%
where $c_t$ denotes a \textit{time-dependent} drift coefficient;
$\sigma>0$ is the noise strength; and $\bbL_\gamma = \bbL + \gamma \bbI$ with $\gamma>0$ being a small centering parameter. 
%
\begin{figure}[t]
    \centering
    \includegraphics[width=\linewidth]{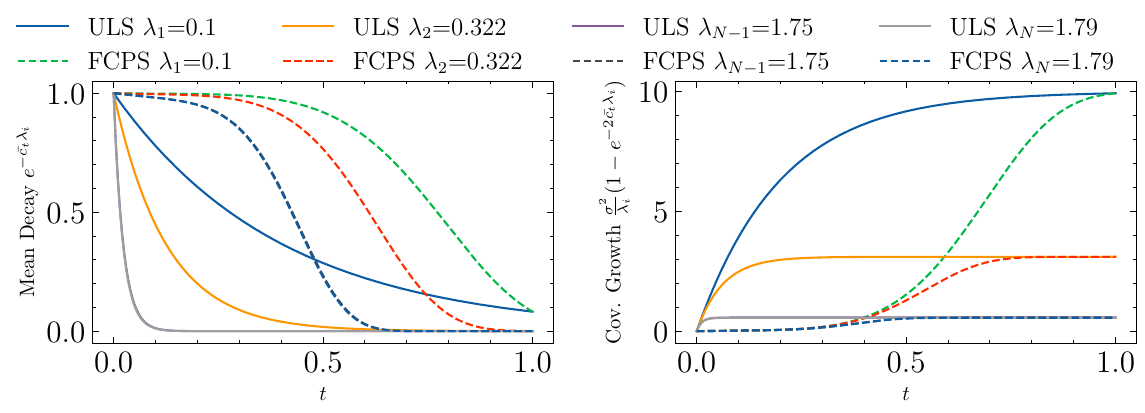}
    \caption{Evolution of the eigenvalues of $\bbH_t$ (left) and $\bbSigma_t$ (right) as time $t$ increases with uniform linear scheduler (ULS) and floor constrained polynomial scheduler (FCPS).}
    \label{fig:eig_decay}
    \vspace{-0.5cm}
\end{figure}
The distributions induced by \eqref{eq::forward} for all $t$ are available in closed form.
Given $\bbx_0$, the distribution of $\bbx_t$ is a GMRF with
\begin{align}
    \label{eq::mean_cov}
    \bbmu_t = \bbH_t \bbx_0, \quad 
    \bbSigma_t = \sigma^2  \left( \bbI - \bbH_t^2 \right) \bbL_\gamma^{-1},
\end{align}
where $\bbH_t = e^{-\bar c_t \bbL_\gamma }$ with $\bar c_t=\int_0^t c_sds$.
Thus, the forward process applies a low-pass filter $\bbH_t$ that attenuates the original signal and injects graph-structured noise.
As $t$ increases, the operator $\bbH_t$ decays exponentially.
Existing works on the graph heat equation for graph-signal generation \cite{zhu2024graph, yang2025topological} assume a uniform linear scheduler for the drift coefficient, i.e.,  $c_t=ct$.
However, this causes the dominant eigenmodes of $-c_t\bbL_\gamma$ to decay too quickly (cf. Figure~\ref{fig:eig_decay}). 
Inspired by time schedulers in image diffusion \cite{nichol2021improved}, we propose a time-dependent floor constrained polynomial scheduler (FCPS) 
\[
    c_t = c_\textnormal{min} + ku^\alpha \quad \text{with} \; u=t/T, \; k=(c_0-c_\textnormal{min}T)(\alpha+1)/T 
\]
and constants $c_0,\alpha>1$, $0<c_\textnormal{min}<1$.
This leads to $\bar c_t = c_\textnormal{min}t + (c_0-c_\textnormal{min}T)u^{\alpha+1}$, which induces a time-warping effect, i.e., for $u\leq 1$, $u^{\alpha+1}$ grows much slower than $u$, especially for $\alpha \gg 1$ (cf. Figure~\ref{fig:eig_decay}).
Regarding $\bbL_\gamma$, because $\bbL$ has a zero eigenvalue, we add $\gamma \bbI$ to ensure $-c_t(\bbL+\gamma\bbI)$ is Hurwitz, i.e. all eigenvalues are strictly negative, which guarantees convergence to a stationary distribution \cite{sarkka2019applied}.
In graph-heat-based generation, the coefficient $\gamma$ is crucial but often overlooked. To guarantee convergence, $\gamma > 0$ is required. Equally important, its value must be chosen carefully: if $\gamma$ is too small, low modes decay too slowly, whereas if it is too large, the geometry induced by the Laplacian is lost.

The SDE in \eqref{eq::forward} converges asymptotically when $\bar c_t \to \infty$ to a stationary distribution given by the Lyapunov equation \cite{sarkka2019applied}, which in this case is a GMRF with $\bbmu_\infty = \bbzero$ and $\bbSigma_\infty = \sigma^2 \bbL_\gamma^{-1}$.
In practice, for $t=T$, we have $\bbH_T = e^{-c_0\bbL_\gamma}$, i.e., $c_0$ controls how close the process gets to the stationary distribution.

%

\begin{remark}
    Since computing matrix exponentials may be numerically unstable, we evaluate $\bbH_t = \bbV e^{-\bar c_t(\bbLambda+\gamma\bbI)}\bbV^\top$ via eigendecomposition. Beyond numerical stability, this approach suggests defining $\bar c_t$ per eigenvalue to independently control mode decay \cite{hoogeboom2023blurring}, a promising future research direction.
\end{remark}

\noindent \textbf{Backward process.}
As per \eqref{eq::diff_back_raw}, the corresponding backward SDE is
\begin{equation}
    \label{eq::backward}
    d\bbx_t = -c_t \bbL_\gamma \bbx_t\,dt - 2c_t\,\sigma^2\,\nabla_{\bbx_t} \log p_t(\bbx_t)\,dt + \sqrt{2c_t}\,\sigma\,d  \bbw_t.
\end{equation}
Because $\nabla_{\bbx_t} \log p_t(\bbx_t)$ is not available in closed form, it must be approximated. Although this is generally challenging, Tweedie’s formula \cite{meng2021estimating} provides a way to rewrite the score as
\begin{equation}
    \label{eq::scores_tweedies}
    \nabla_{\bbx_t} \log p_t(\bbx_t)
    = \bbSigma_t^{-1} \left( \bbH_t \, \mathbb{E}[\bbx_0 \mid \bbx_t] - \bbx_t \right).
\end{equation}
Recall that $\bbH_t$ and $\bbSigma_t$ are both known. Interestingly, the posterior mean $\mathbb{E}[\bbx_0 \mid \bbx_t]$ can be reformulated as the minimizer $g^\star$ of the minimum mean squared (MMSE) problem \cite{meng2021estimating}
\begin{equation}
    \label{eq::mmse_error}
    g^\star \in \argmin_g \;\; \ell(g) := \mathbb{E}\left[ \| \bbx_0 - g(\bbx_t)\|^2 \right],
\end{equation}
so that $\mathbb{E}[\bbx_0 \mid \bbx_t] = g^\star(\bbx_t)$.
Therefore, by solving \eqref{eq::mmse_error} and obtaining a minimizer $g^\star$, we can evaluate the scores $\nabla_{\bbx_t} \log p_t(\bbx_t)$ from \eqref{eq::scores_tweedies} and use them to run the backward process in \eqref{eq::backward}. 
While solving \eqref{eq::mmse_error} is nontrivial, the problem reduces to identifying a mapping $g$ that removes noise from a graph signal, a classical problem in GSP.
Hence, in next section we adopt a GSP perspective to address \eqref{eq::mmse_error}.

\section{A Graph-Signal-Processing Perspective}

Designing a mapping $g$ to recover $\bbx_0$ from $\bbx_t$ constitutes a \emph{graph-based inverse problem}. We adopt a GSP perspective to design an estimator that explicitly accounts for the underlying graph structure. Specifically, the noise model is
\begin{equation}
    \label{eq::noise_model}
    \bbx_t = \bbH_t\bbx_0 + \bbe, \quad \bbe \sim \ccalN(\bbzero, \bbSigma_t),
\end{equation}
where $\bbH_t$ and $\bbSigma_t$ are both known and defined in \eqref{eq::mean_cov}. 
Recovery amounts to removing the noise $\bbe$ and deconvolving to obtain $\bbx_0$, a problem extensively studied in the GSP literature \cite{do2020graph, fu2020understanding, chen2021graph, nagahama2022graph, kroizer2022bayesian, rey2023robust}.


\begin{figure}[t]
    \centering
    \includegraphics[width=\linewidth]{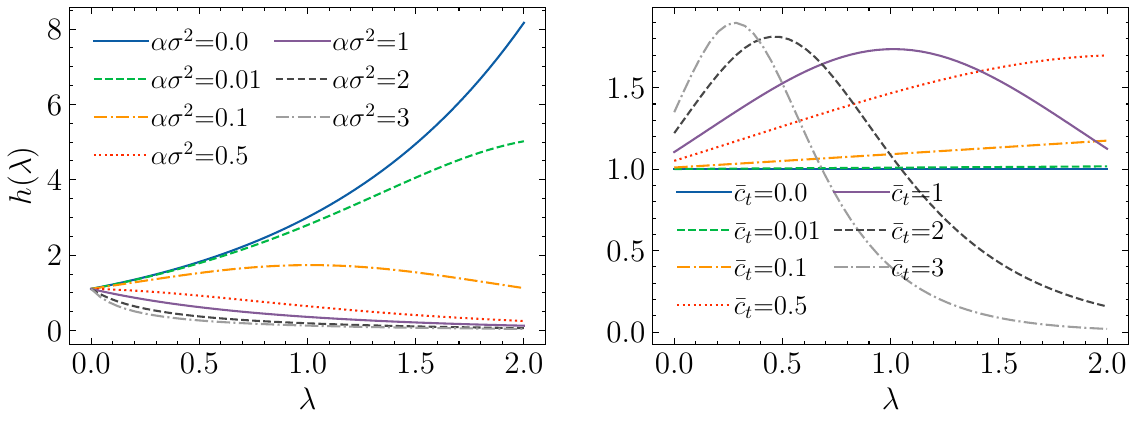}
    \caption{Frequency response of the graph filter in \eqref{eq::gf_solution} for different values of $\alpha \sigma^2$ (left) and $\bar c_t$ (right). }
    \label{fig:freq_response}
    \vspace{-0.5cm}
\end{figure}

\noindent \textbf{Graph-filter estimator.}
To recover $\bbx_0$ from the noise model in \eqref{eq::noise_model}, one typically leverages prior structural knowledge of $\bbx_0$ to simplify the denoising task.
We exploit the fact that many graph signals are smooth with respect to the graph Laplacian, as extensively studied in the graph-signal denoising literature \cite{ramakrishna2020user}.
To gain insight into the graph-denoising problem, we first consider a single signal $\bbx_0$ at a fixed time $t$, where the mapping $g$ can be defined as the optimization
\begin{equation}
    \label{eq::denoising_problem}
    \hat \bbx_0 = g(\bbx_t) = \arg\min_{\bbx_0} \;\;
    \| \bbx_t - \bbH_t \bbx_0  \|_{\bbSigma_t^{-1}}^2 + \alpha \| \bbx_0 \|_\bbL^2,
\end{equation}
where $\alpha \geq 0$ is a regularization constant. 
Unlike standard graph denoising, both the low-pass graph-filter $\bbH_t$ and the noise covariance $\bbSigma_t$ are known, so they can be directly incorporated into the optimization.
The solution to \eqref{eq::denoising_problem} admits closed-form expression, that we summarize in the following proposition.
\begin{proposition}
    \label{prop::graph_filter}
    The solution to \eqref{eq::denoising_problem} is given by $\hat \bbx_0 = h(\bbL) \bbx_t$, with
    \begin{equation}
        \label{eq::solution}
        h(\bbL) = \left( \bbH_t^\top \bbSigma_t^{-1} \bbH_t + \alpha \bbL \right)^{-1} \bbH_t^\top \bbSigma_t^{-1},
    \end{equation}
    which is a graph filter in the eigenbasis of $\bbL$ with frequency response
    \begin{equation}
    h(\lambda) = \frac{e^{(\lambda + \gamma)\bar c_t}}{1 + \alpha \sigma^2 \frac{\lambda}{\lambda+\gamma}(e^{2(\lambda + \gamma) \bar c_t}-1)}.
\end{equation}
\end{proposition}
\begin{proof}
    Taking gradients of \eqref{eq::denoising_problem} with respect to $\bbx_0$ and setting them to zero yields the expression in \eqref{eq::solution}.
    Next, to prove that it corresponds to a graph filter. Note that $\bbL_\gamma$ is diagonalizable by the same eigenbasis as $\bbL$, i.e., $\bbL_\gamma = \bbV \bbLambda_\gamma \bbV^\top$ with $\bbLambda_\gamma = \bbLambda + \gamma \bbI$. Then, by the properties of the matrix exponential, we have that
    \begin{align}
        \label{eq::h_eigen}
        \bbH_t &= e^{-\bar c_t \bbL_\gamma} = \bbV e^{-\bar c_t \bbLambda_\gamma} \bbV^\top \\
        \label{eq::sigma_eigen}
        \bbSigma_t &= \sigma^2 \!\! \left( \bbI \!-\! e^{-2 \bar c_t \bbL_\gamma} \! \right) \bbL_\gamma^{-1} \!\! = \!\! \bbV \! \left[ \sigma^2 \!\! \left( \bbI \!-\! e^{-2 \bar c_t \bbLambda_\gamma} \! \right)\bbLambda_\gamma^{-1}\right] \!\! \bbV^\top \!\!.
    \end{align}
    Defining
    $\bbLambda_{\bbH_t}=e^{-\bar c_t \bbLambda_\gamma}$, and $\bbLambda_{\bbSigma_t}=\sigma^2 \left( \bbI - e^{-2 \bar c_t \bbLambda_\gamma} \right)\bbLambda_\gamma^{-1}$,
    and substituting \eqref{eq::h_eigen} and \eqref{eq::sigma_eigen} into \eqref{eq::solution} yields
    \begin{align}
        \notag
        h(\bbL) &= \left( \bbH_t^\top \bbSigma_t^{-1} \bbH_t + \alpha \bbL \right)^{-1} \bbH_t^\top \bbSigma_t^{-1} \\
        \label{eq::gf_solution}
        &= \bbV \left[ \left( \bbLambda_{\bbH_t}^2 \bbLambda_{\bbSigma_t}^{-1} + \alpha \bbLambda \right)^{-1} \bbLambda_{\bbH_t} \bbLambda_{\bbSigma_t}^{-1}\right] \bbV^\top,
    \end{align}
    which is a graph filter in the eigenbasis of $\bbL$.
\end{proof}

Proposition \ref{prop::graph_filter} shows that the solution to \eqref{eq::denoising_problem} has a rational frequency response $h(\lambda)$ with the form of an ARMA graph filter~\cite{isufi2016autoregressive} and provides insights into the behavior of the proposed DM.
%
%
In particular, $h(\lambda)$ reflects a tension between the exponential numerator, which replenishes high frequencies, and the doubly-exponential denominator, which suppresses them. 
The factor $\alpha\sigma^2$ controls how strongly the denominator weighs in. As shown in Fig.~\ref{fig:freq_response} (left), when $\alpha\sigma^2$ is small, the numerator dominates and the filter behaves like a high-pass operator. In contrast, for large $\alpha\sigma^2$, the denominator prevails, yielding a low-pass response. The time parameter $\bar c_t$ plays a particularly interesting role: although both terms grow with $\bar c_t$, the denominator grows faster. Consequently, as seen in Fig.~\ref{fig:freq_response} (right), large $\bar c_t$ inevitably produces a low-pass response, while decreasing $\bar c_t$ transitions the filter from high-pass to nearly flat when the forward process introduces little noise. This progression suggests that the backward process unfolds in stages: at large $\bar c_t$ it primarily recovers coarse, low-frequency structure, while smaller $\bar c_t$ gradually reintroduces high-frequency details.



\begin{algorithm}[t]
\vspace{-1mm}
\caption{Training of GAD}\label{alg:training}
    \KwIn{Graph Laplacian $\bbL$; data samples $\ccalD = \{\bbx_0^{(i)}\}_{i=1}^M$; time horizon $T$; and learning rate $\eta$.}
    
    Sample $\bbx_0 \sim \ccalD$ and $t \sim \mathcal{U}(0,T)$
    
    Compute $\bbmu_t,\bbSigma_t$ via \eqref{eq::mean_cov} and draw $\bbx_t \sim \mathcal{N}(\bbmu_t,\bbSigma_t)$
    
    Estimate $\hat{\bbx}_0$ using the GCNN in \eqref{eq::single_layer_of_gnn} as $\hat{\bbx}_0 \gets g_{\bbTheta}(\bbx_t,\bbL, t)$

    \vspace{-1mm}
    Update $\bbTheta$ via gradient descent on \eqref{eq::gnn_training}  as $\bbTheta \!\leftarrow\! \bbTheta \!-\! \eta \hat \nabla_{\bbTheta}\ell(\bbTheta)$\;
\end{algorithm}

\begin{figure*}[t]
    \centering
    \includegraphics[width=\linewidth]{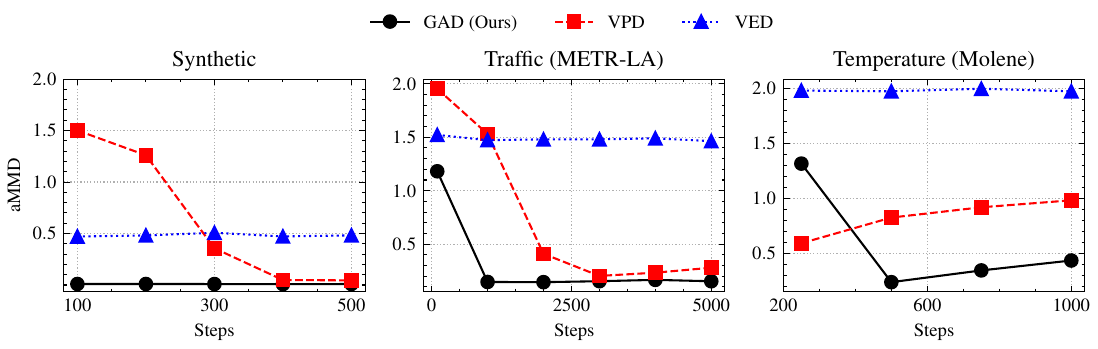}
    \caption{GAD outperforms VPD and VED, particularly when only a small number of discretization steps are used. 
    This trend holds consistently in both the synthetic SBM setting (left), the METR-LA traffic dataset (middle), and the Molene temperature dataset (right).}
    \vspace{-0.5cm}
    \label{fig:results}
\end{figure*}

\noindent \textbf{GCNN estimators.}
While \eqref{eq::solution} provides a principled way to recover $\bbx_0$ from $\bbx_t$, it applies only to a single signal and timestep.
In practice, problem \eqref{eq::mmse_error} considers $\bbx_0$ in expectation, and the backward process requires estimating $\bbx_0$ across different values of $t$, with the associated filter evolving from low-pass in the early stages to high-pass toward the end.
The interpretation of \eqref{eq::solution} as a graph filter in the eigenbasis of $\bbL$ motivates the design of a learnable graph filter architecture to denoise $\bbx_t$ across varying timesteps and signals. Building on this idea, our GAD architecture proposes using a GCNN~\cite{gama2020graphs}, consisting of a cascade of $L$ graph filters, as a flexible and learnable graph-aware denoiser. Specifically, each layer $l$ computes
\begin{equation}
\label{eq::single_layer_of_gnn}
    \hat \bbx^{(l+1)} = \text{ReLU}\left( \sum_{k=0}^K \theta_k^{(l)} \bbL^k \hat \bbx^{(l)} \right),
\end{equation}
where $\bbtheta^{(l)} = [\theta_0^{(l)}, \ldots, \theta_K^{(l)}]^\top$ are the filter coefficients at layer $l$, and $\sigma$ is a point-wise nonlinearity. 
Stacking all layers, the GCNN is defined as a mapping $g_{\bbTheta}(\bbx_t, \bbL, t)$ with $\bbTheta = [\bbtheta^{(0)}, \ldots, \bbtheta^{(L)}]$, where the time index is incorporated as a graph feature \cite{yang2023diffusion}.
Note that $g_{\bbTheta}$ parametrizes $g$, yielding the approximate MMSE problem
\begin{equation}
\label{eq::gnn_training}
    \bbTheta^\star
    \in \argmin_{\bbTheta} \; \ell(\bbTheta):=\mathbb{E} \left[ || \bbx_0 - g_{\bbTheta}(\bbx_t, \bbL, t) ||^2 \right],
\end{equation}
so that $\mathbb{E}[\bbx_0 \mid \bbx_t] \approx g_{\bbTheta^\star}(\bbx_t, \bbL, t)$.
Problem \eqref{eq::gnn_training} is solved via stochastic gradient descent over varying noise levels as depicted in  Algorithm~\ref{alg:training}, which summarizes our main contribution.
Once trained, the estimator $g_{\bbTheta}(\bbx_t, \bbL, t)$ recovers the unknown signal $\bbx_0$ from any observation $\bbx_t$, which in turn enables the computation of approximate scores
\[
    \nabla_{\bbx_t} \log p_t(\bbx_t)
    \approx \bbSigma_t^{-1} \left( \bbH_t \, g_{\bbTheta^\star}(\bbx_t, \bbL, t) - \bbx_t \right).
\]
These approximate scores can be used to simulate the backward process in \eqref{eq::backward}, which can be implemented with different SDE solvers. 
We use the Euler–Maruyama solver, while noting that leveraging advanced solvers represents a promising future research direction.

\section{Numerical experiments}

We evaluate GAD on: i) a synthetic scenario based on a stochastic block model (SBM) graph with smooth signals; ii) METR-LA \cite{li2018diffusion}, a real traffic dataset of highway speed measurements, and iii) Molene \cite{girault2017local}, a real dataset of temperature measurements from sensors in France.
The code for these experiments is available on GitHub \footnote{ \url{https://github.com/vimalkb7/gad}}.

\noindent \textbf{Datasets.}
In the synthetic example we consider an SBM graph with two communities of $10$ nodes each. 
We generate $500$ train graph signals and $500$ test graph signals, where each community follows a Gaussian distribution with standard deviation $1$, mean $1$ in one community, and mean $-1$ in the other. 
To induce smoothness, the signals are filtered with a low-pass graph filter. 
The METR-LA dataset contains over $200{,}000$ traffic speed measurements from $207$ loop detectors in Los Angeles. The graph is defined by the road network, with sensors as nodes and edges representing physical adjacency. We randomly sample $2,000$ observations for training and $500$ for testing. 
The Molene dataset contains temperature measurements from $37$ weather stations in France. The graph is defined by proximity between stations, with nodes as weather stations and edges encoding spatial adjacency. We use daily average temperatures as signals, with $670$ observations for training and $74$ for testing.

\noindent \textbf{Experimental setup.} Since most of the literature is application-specific, selecting baselines is not straightforward.
Most existing approaches fall under the variance-preserving diffusion framework \cite{daiya2024diffstock, uslu2025generative}, which can be unified under the SDE perspective (VPD) \cite{yang2023diffusion}. We also include the variance-exploding diffusion framework (VED) as a baseline, which has not been used for graph-signal diffusion but is common in image-based diffusion \cite{yang2023diffusion}. Both frameworks are graph-agnostic, which under the SDE perspective makes them directly comparable to GAD.
For evaluation, we follow the procedure used in graph generation \cite{liu2023generative}.
We first extract three metrics from both the generated and test signals that capture their statistical, structural, and spectral properties. 
We then compute the maximum mean discrepancy (MMD) between the distributions of each metric in the generated and test sets. Finally, we average these three MMD values to obtain the average MMD (aMMD), which quantifies the similarity between the generated and test signal distributions.
Regarding the metrics, we employ the quadratic variation (QV) to capture the smoothness, the spectral centroid (SC) to summarize the distribution of the spectrum, and the degree correlation (DC) to measure how signals align with the degree vector of the graph. Formally we have
\begin{align}
    \notag
    &\mathrm{QV}(\bbx) = \bbx^\top \bbL \bbx, \quad \mathrm{SC}(\bbx) = \frac{\sum_{i=1}^N \lambda_i \, |\tilde x_i|^2}{\sum_{i=1}^N |\tilde x_i|^2},  \\
    \notag
    &\text{and} \quad \mathrm{DC}(\bbx) =
    \frac{(\bbx - \bar \bbx )^\top (\bbd - \bar \bbd )}
    {\|\,\bbx - \bar \bbx \,\|_2 \;\;\|\,\bbd - \bar \bbd \,\|_2} 
\end{align}
where $\tilde \bbx = \bbV^\top \bbx$ and $\bar \bbx = (\tfrac{1}{N}\mathbf{1}^\top \bbx) \bbone$ and $\bar \bbd = (\tfrac{1}{N}\mathbf{1}^\top \bbd) \bbone$.
As stated previously, generation is carried out with an Euler–Maruyama discretization of \eqref{eq::backward}. 
Increasing the number of discretization steps improves quality but raises computational cost, especially in DM models where each step requires a forward pass through the GCNN, which can be prohibitive.  We report results for different step counts.

\noindent \textbf{Findings.}  The results are illustrated in Fig.~\ref{fig:results}.
Across datasets GAD consistently outperforms the baselines, with its advantage most pronounced when the number of steps is small. 
On the synthetic SBM graph this translates into clear gains over both VPD and VED at low step counts, with GAD and VPD approaching similar performance as the number of steps grows. 
On Molene, GAD outperforms VPD and VED overall, but the dependence on the number of steps is less pronounced. VPD can be competitive for very few steps, yet its performance deteriorates as the step count increases.
Regarding METR-LA, the behavior mirrors the synthetic case, with GAD outperforming the alternatives, especially at small step counts.
Overall these results show that GAD exploits the graph structure more effectively than existing approaches, delivering consistent improvements across datasets and particularly strong benefits when a limited number of discretization steps can be used.

\section{Conclusions}

This paper presented GAD, a graph-aware generative DM for graph signals, building on DMs based on the heat equation. 
We introduced a forward process defined by the graph Laplacian with a time-dependent drift factor controlling the decay rate, and characterized its limiting distribution as a GMRF with covariance parametrized by the graph.
We analyzed the backward process from a GSP perspective, showing that score estimation reduces to a sequence of graph-signal denoising problems, and studied the behavior of the associated optimal filters through spectral analysis. 
This motivated the use of GCNNs as approximators within the model. 
Experiments with synthetic and real data showed that GAD outperforms graph-agnostic baselines, especially for small number of generative steps.

\newpage

\bibliographystyle{IEEEtran}
\bibliography{refs}

\end{document}